\documentclass{article} 
\usepackage[final]{colm2026_conference}

\usepackage{microtype}
\usepackage{hyperref}
\usepackage{url}
\usepackage{booktabs}
\usepackage{multirow}
\usepackage{amsfonts}
\usepackage{amsmath,amssymb,amsthm}
\usepackage{nicefrac}
\usepackage{xcolor}
\usepackage{algorithm}
\usepackage{algorithmic}
\usepackage{graphicx}
\usepackage{subcaption}
\usepackage{tikz}
\usetikzlibrary{arrows.meta, positioning, shapes.geometric}

\usepackage{lineno}

\definecolor{darkblue}{rgb}{0, 0, 0.5}
\hypersetup{colorlinks=true, citecolor=darkblue, linkcolor=darkblue, urlcolor=darkblue}

\newtheorem{proposition}{Proposition}

\title{Se-DPO: Self-Evolving Token Credit for Direct Preference Optimization}

\author{%
Wenxiao Zhao$^{1,2}$ \quad Shu Wang$^{1}$ \quad Ying Nian Wu$^{1}$\\
$^{1}$University of California, Los Angeles\\
$^{2}$Shanghai AI Laboratory%
}

\begin{document}

\ifcolmsubmission
\linenumbers
\fi
\maketitle

\begin{abstract}

Direct Preference Optimization (DPO) aggregates token-level log-probability ratios via uniform summation, implicitly treating all tokens as contributing equally to the preference signal. However, the contribution of individual tokens to the preference signal varies. We introduce token credit, which modulates each token's KL regularization based on its contribution to the preference outcome. We derive that effective token credit is proportional to the magnitude of each token's implicit reward, and observe that this quantity evolves substantially during training. This implies that static token credit becomes increasingly misaligned as training progresses. In this work, we propose Se-DPO (Self-Evolving Token Credit for DPO), a live mechanism that derives token credit from the model's own evolving internal signals during DPO training. Since the reward signal varies in reliability across positions, Se-DPO calibrates token credit based on both the strength and the confidence of each token's contribution. Se-DPO requires no external models, adding only a lightweight calibration network with minimal computational overhead. Experiments show that Se-DPO improves over DPO by up to 9.8 points on AlpacaEval~2 and 12.2 points on Arena-Hard.
\end{abstract}

\section{Introduction}
 
Direct Preference Optimization \citep{dpo} has emerged as a widely adopted method for aligning large language models with human preferences. By reparameterizing the reward function in the reinforcement learning from human feedback (RLHF) objective \citep{christiano2017deep,ziegler2019finetuning,ouyang2022training,stiennon2020learning,bai2022training,touvron2023llama2,schulman2017ppo}, DPO enables direct policy optimization from preference data without training a separate reward model. In the standard formulation, DPO aggregates token-level log-probability ratios via uniform summation, implicitly treating all tokens as equally important to the preference signal.
 
However, not all tokens contribute equally to the quality of a response. A factual error in a single token can flip a preference label \citep{rafailov2024token}, while filler words rarely affect human judgment. Recent methods have sought to address this mismatch by introducing token-level importance signals into the preference optimization objective, drawing on external sources such as pre-trained teacher models \citep{tgdpo,qrm,rto}, contrastive language model pairs \citep{tisdpo}, learned sparse masks \citep{sparsepo}, optimal transport matching \citep{otpo}, per-token KL constraints \citep{tdpo}, prompted self-evaluation \citep{treg}, or oracle-based token selection \citep{sepo}. While these approaches differ in their signal sources, they share a common pattern: token importance is computed prior to or outside the training loop and treated as approximately fixed throughout optimization. A natural alternative is to derive token credit from DPO's own implicit reward \citep{rafailov2024token}, but recent work observed that this signal can assign disproportionate importance to non-critical tokens such as punctuation and line breaks \citep{qrm}, leaving its direct use largely unexplored.
 
In this work, we study the implicit reward not at a single snapshot but across the full trajectory of DPO training. Our analysis reveals that the token-level implicit reward undergoes substantial evolution during training: the set of tokens deemed important shifts considerably between early and late stages, and the importance ranking stabilizes only gradually over the course of optimization. Crucially, this gradual stabilization suggests that the noisy importance estimates observed by prior work~\citep{qrm} reflect the early training state rather than an inherent limitation of the implicit reward. Since token importance continues to shift until convergence, any static credit assignment becomes increasingly stale as optimization progresses. This temporal perspective distinguishes our approach from methods that adapt KL regularization at the sample level \citep{epsilon-dpo,wpo,ivison2024unpacking} or bootstrap implicit rewards at the sequence level \citep{dice}---neither of which captures the within-training evolution of token-level importance.
 
Based on these observations, we show that relaxing DPO's uniform KL coefficient to be token-specific yields an objective where each token's contribution is scaled by a credit term. We derive that, under a variance-minimization criterion on the preference logit, the credit at each position should be proportional to the magnitude of the token's implicit reward. Intuitively, tokens with larger implicit rewards carry more preference-relevant information and should be allowed to deviate further from the reference policy, while tokens with near-zero implicit rewards contribute little to the preference signal and benefit from tighter regularization. Since the implicit reward evolves during training, the credit should also evolve---any static approach uses an increasingly stale signal by design.
 
We propose \textbf{Se-DPO} (Self-Evolving Token Credit for DPO), a live mechanism that periodically extracts implicit reward magnitudes as token credit during DPO training. Se-DPO adds no external models and only a lightweight calibration network with minimal computational overhead.
 
Our contributions are as follows:
\begin{itemize}
    \item We observe that DPO's token-level implicit reward shifts considerably over the course of training, with limited overlap between top-$|\hat r_t|$ tokens at early and late stages.
    \item We introduce token credit to modulate per-token KL regularization strength, and analyze how it can be informed by implicit reward magnitude and estimation confidence.
    \item We propose Se-DPO, which derives token credit online from the model's own evolving internal signals via a lightweight calibration network, requiring no external models and achieving up to 50.6\% win rate on AlpacaEval~2 and 43.3\% on Arena-Hard.
\end{itemize}

\section{Preliminaries}
\label{sec:prelim}
\subsection{DPO and Implicit Reward}

The standard RLHF objective maximizes expected reward subject to a KL divergence constraint against a reference policy $\pi_{\text{ref}}$:
\begin{equation}
    \max_{\pi_\theta} \mathbb{E}_{x \sim \mathcal{D}, y \sim \pi_\theta}\left[r(x, y)\right] - \beta \, D_{\text{KL}}\!\left[\pi_\theta(\cdot|x) \,\|\, \pi_{\text{ref}}(\cdot|x)\right],
    \label{eq:rlhf}
\end{equation}
where $\beta > 0$ controls the strength of regularization. DPO \citep{dpo} reparameterizes the reward in terms of the optimal policy, yielding a supervised loss over preference pairs $(x, y^c, y^r)$:
\begin{equation}
    \mathcal{L}_{\text{DPO}} = -\mathbb{E}_{(x, y^c, y^r)}\!\left[\log\sigma\!\left(\beta\log\frac{\pi_\theta(y^c|x)}{\pi_{\text{ref}}(y^c|x)} - \beta\log\frac{\pi_\theta(y^r|x)}{\pi_{\text{ref}}(y^r|x)}\right)\right].
    \label{eq:dpo}
\end{equation}
Since language models generate tokens autoregressively, the sequence-level log-ratio decomposes as a sum of token-level terms. \citet{rafailov2024token} showed that DPO implicitly learns a token-level reward:
\begin{equation}
    \hat{r}_t = \beta\!\left(\log\pi_\theta(y_t|x, y_{<t}) - \log\pi_{\text{ref}}(y_t|x, y_{<t})\right).
    \label{eq:implicit-reward}
\end{equation}

\begin{figure*}[t]
\centering
\includegraphics[width=\textwidth]{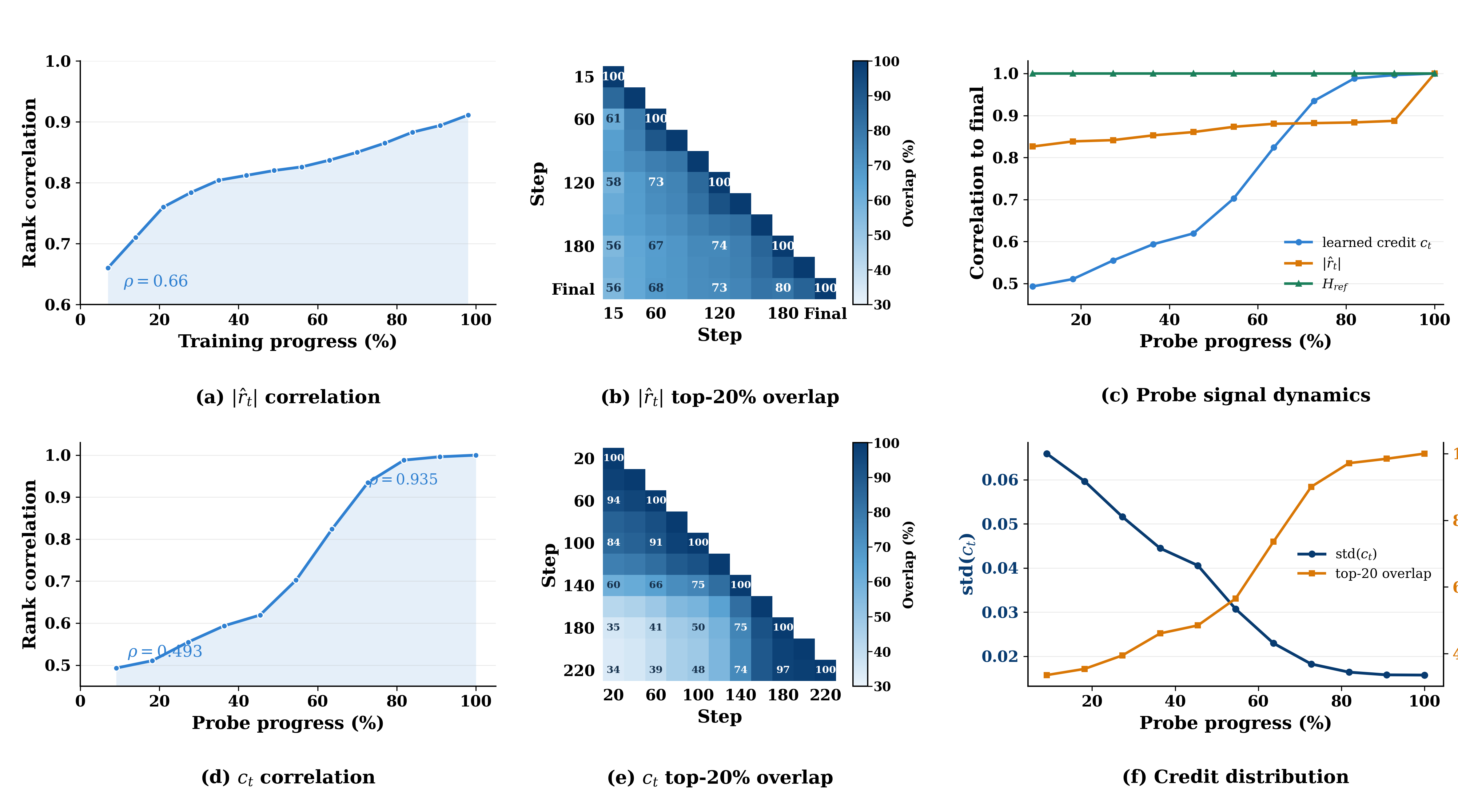}
\caption{
Token importance and learned credit evolve during training.
\textbf{(a,b)} Implicit-reward magnitude $|\hat r_t|$ in the original training run: rank correlation to the final checkpoint increases over time, while top-20\% token overlap remains limited.
\textbf{(c--f)} Learned credit $c_t$ in a separate Se-DPO LoRA credit-dynamics probe: credit correlation increases from 0.493 to 1.0, early-to-final top-20\% overlap is 33.6\%, and the credit distribution gradually stabilizes.
}
\label{fig:credit_dynamics}
\end{figure*}

\subsection{Token-Specific KL Regularization}
\label{sec:token-kl}

Standard DPO applies a uniform KL coefficient $\beta$ to all token 
positions. Prior work has explored token-level KL 
decompositions~\citep{tdpo} and learned masks over token-level KL 
contributions~\citep{sparsepo}. In this work, we consider a 
generalization where the KL strength varies across positions:
\begin{equation}
    \max_\pi \mathbb{E}_{y \sim \pi}\!\left[\sum_{t=1}^{T} r_t 
    - \beta_t \log\frac{\pi(y_t | x, y_{<t})}
    {\pi_{\text{ref}}(y_t | x, y_{<t})}\right],
    \label{eq:token-obj}
\end{equation}
where $\beta_t$ controls the KL penalty at position $t$. At each 
position, the optimization problem is:
\begin{equation}
    \max_{\pi(\cdot|x,y_{<t})} \sum_{y_t} 
    \pi(y_t|x,y_{<t})\!\left[r_t(y_t) 
    - \beta_t \log\frac{\pi(y_t|x,y_{<t})}
    {\pi_{\text{ref}}(y_t|x,y_{<t})}\right].
    \label{eq:per-position-opt}
\end{equation}
This is a KL-regularized optimization whose closed-form solution is:
\begin{equation}
    \pi^*(y_t|x,y_{<t}) = \frac{1}{Z_t}\pi_{\text{ref}}(y_t|x,y_{<t})
    \exp\!\left(\frac{r_t(y_t)}{\beta_t}\right),
    \label{eq:optimal-policy-token}
\end{equation}
where $Z_t = \sum_{y_t'}\pi_{\text{ref}}(y_t'|x,y_{<t})
\exp(r_t(y_t')/\beta_t)$ is a position-specific partition function. 
Rearranging yields the implicit reward:
\begin{equation}
    r_t(y_t) = \beta_t\!\left(\log\pi^*(y_t|x,y_{<t}) 
    - \log\pi_{\text{ref}}(y_t|x,y_{<t})\right) + \beta_t\log Z_t.
    \label{eq:implicit-reward-token}
\end{equation}
Since $Z_t$ depends only on the context $(x, y_{<t})$ and not on the 
sampled token $y_t$, it does not affect the relative contribution of 
different tokens within a response.%
\footnote{With position-dependent $\beta_t$, the per-position partition
functions generally do not cancel between chosen and rejected responses,
even when $\beta_t$ is fixed: the two continuations induce different
contexts after their first divergence, and therefore different $Z_t$ values.
Eq.~\ref{eq:weighted-logit} should thus be understood as a principled
approximation to the token-specific KL objective. The uncancelled term acts
as a context-dependent value offset; because Se-DPO mean-normalizes credits
within each response and keeps them bounded, this residual cannot grow by
uniformly scaling token-credit values. Similar approximations are also used by
other token-level objectives that modify the per-token KL structure.}
Define $c_t = \beta / \beta_t$ and recall the token-level implicit 
reward $\hat{r}_t^{(s)} = \beta(\log\pi_\theta(y_t^{(s)}|x, 
y_{<t}^{(s)}) - \log\pi_{\text{ref}}(y_t^{(s)}|x, y_{<t}^{(s)}))$ 
from Eq.~\ref{eq:implicit-reward}. The resulting preference logit is:
\begin{equation}
    \hat{\Delta} = \sum_{t=1}^{T_c} c_t^{(c)}\, \hat{r}_t^{(c)} 
    - \sum_{t=1}^{T_r} c_t^{(r)}\, \hat{r}_t^{(r)},
    \label{eq:weighted-logit}
\end{equation}
where standard DPO is recovered when $\beta_t = \beta$ (i.e., $c_t = 1$) 
for all $t$. A smaller $\beta_t$ at position $t$ yields a larger credit 
$c_t$, reducing the KL penalty and allowing the policy to deviate further 
from the reference at that position. The central question is how to set 
$\beta_t$---equivalently, how to allocate per-token KL budgets.
\begin{figure*}[t]
\centering
\begin{minipage}[b]{0.32\linewidth}
    \centering
    \includegraphics[width=\linewidth]{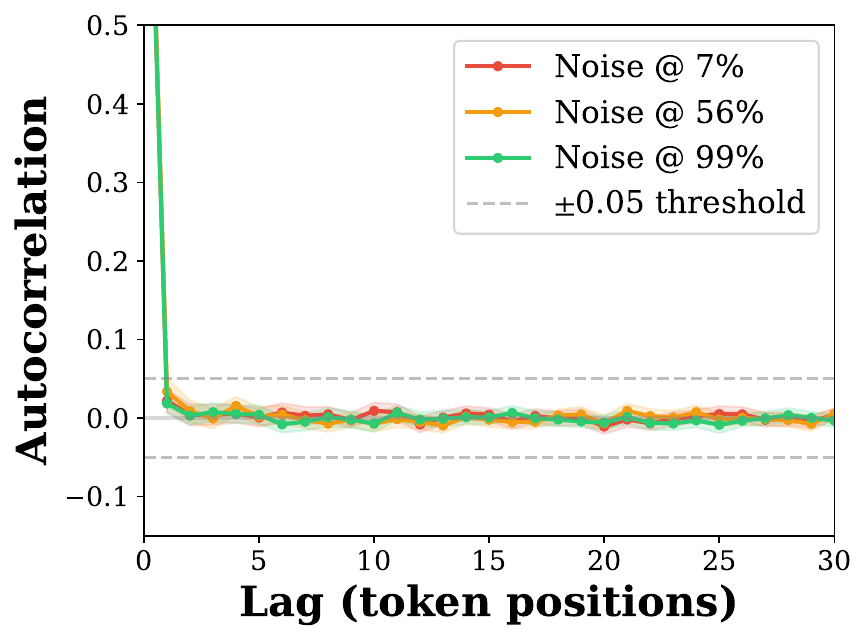}
    \subcaption{Noise autocorrelation decay.}
    \label{fig:val_autocorr}
\end{minipage}
\hfill
\begin{minipage}[b]{0.32\linewidth}
    \centering
    \includegraphics[width=\linewidth]{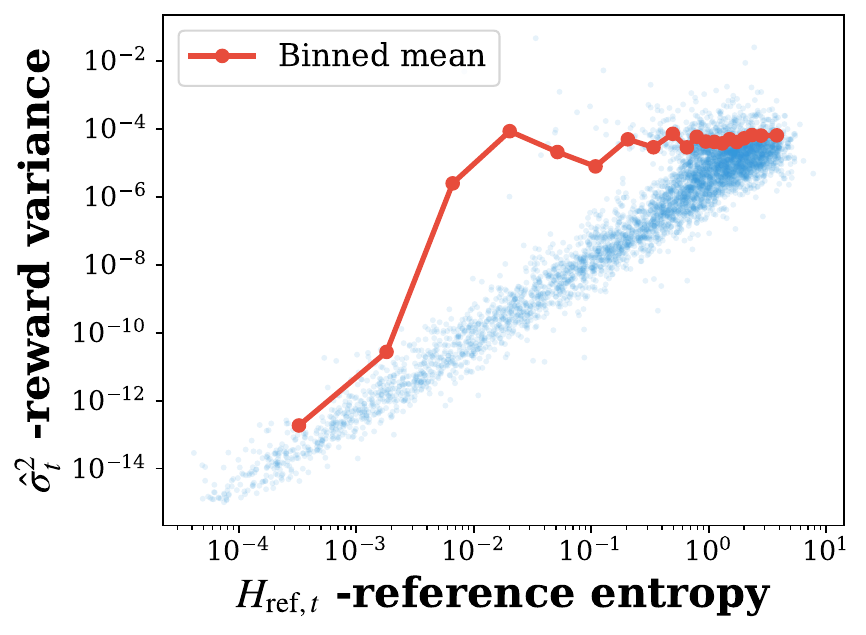}
    \subcaption{$H_{\mathrm{ref},t}$ vs.\ $\hat{\sigma}_t^2$.}
    \label{fig:val_scatter}
\end{minipage}
\hfill
\begin{minipage}[b]{0.32\linewidth}
    \centering
    \includegraphics[width=\linewidth]{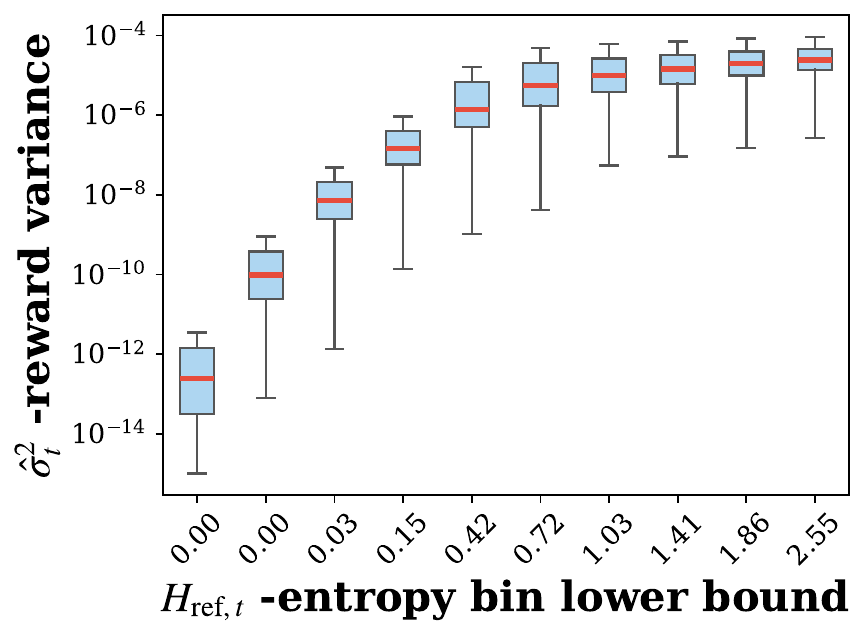}
    \subcaption{Variance by entropy decile.}
    \label{fig:val_boxplot}
\end{minipage}

\caption{Empirical validation of Proposition~\ref{prop:optimal-weight} assumptions.
\textbf{(a)} Noise autocorrelation of $\hat{r}_t$ drops below 0.05 within lag~1 
at all training stages ($m/T < 0.004$), confirming near-independence across positions.
\textbf{(b)} Log-log scatter of reference entropy vs.\ empirical reward variance 
(Spearman $\rho = 0.86$); binned means (red) show a clear monotonic trend.
\textbf{(c)} Boxplot of $\hat{\sigma}_t^2$ by entropy decile confirms that 
higher-entropy positions produce less reliable implicit rewards.}
\label{fig:validation}
\end{figure*}

\section{Method}
\label{sec:method}
 
\subsection{Token Credit Dynamics During Training}
\label{sec:motivation}

Existing token-level methods compute importance prior to or outside 
the training loop, treating it as approximately fixed. However, it 
remains unclear how token importance behaves as training progresses. 
To investigate, we track the implicit reward 
(Eq.~\ref{eq:implicit-reward}) across the trajectory of standard 
DPO training (details in Appendix~\ref{app:details}).

\paragraph{Token importance evolves continuously during training.}
Figure~\ref{fig:credit_dynamics} provides empirical motivation for our self-evolving credit design. We first revisit the implicit reward magnitude $|\hat{r}_t|$, which reflects token-level preference signals. As shown in Figure~\ref{fig:credit_dynamics}(a,b), top-$|\hat r_t|$ tokens change continuously during the original training run: their rank correlation with the final checkpoint increases over time, while the early-to-final top-20\% overlap remains limited.

To verify that this phenomenon also appears in the actual credit used by Se-DPO, we run a separate Se-DPO LoRA credit-dynamics probe on fixed evaluation samples. As shown in Figure~\ref{fig:credit_dynamics}(c--f), the learned credit $c_t$ used for per-token KL-budget allocation is also dynamic. Its correlation with the final checkpoint starts from 0.493 and gradually reaches 1.0, while the early-to-final top-20\% overlap is only 33.6\%. In contrast, the reference entropy $H_{\mathrm{ref}}$ remains static. These observations suggest that token credit is shaped by the evolving model state, motivating our online credit update rather than a static per-token KL-budget allocation. We formalize this intuition in Section~\ref{sec:signals} and propose the complete Se-DPO algorithm in Section~\ref{sec:Se-DPO}.
 
\begin{algorithm}[t]
\caption{Se-DPO: Self-Evolving Token Credit Direct Preference Optimization}
\label{alg:Se-DPO}
\begin{algorithmic}[1]
\REQUIRE Data $\mathcal{D}$, policy $\pi_\theta$, reference $\pi_{\text{ref}}$, calibration network $f_\phi$, warmup steps $W$
\FOR{each training step with batch $(x, y^c, y^r) \sim \mathcal{D}$}
    \STATE \textit{// Token-level implicit rewards}
    \FOR{$s \in \{c, r\}$, $t = 1, \ldots, T_s$}
        \STATE $\hat{r}_t^{(s)} \leftarrow \beta\!\left(\log\pi_\theta(y_t^{(s)} \mid x, y_{<t}^{(s)}) - \log\pi_{\text{ref}}(y_t^{(s)} \mid x, y_{<t}^{(s)})\right)$
    \ENDFOR
    \STATE \textit{// Token credit assignment}
    \IF{step $\leq W$}
        \STATE $c_t^{(s)} \leftarrow 1$ for all $t, s$
    \ELSE
        \STATE $H_{\text{ref},t}^{(s)} \leftarrow -\sum_{v} \pi_{\text{ref}}(v \mid x, y_{<t}^{(s)}) \log \pi_{\text{ref}}(v \mid x, y_{<t}^{(s)})$ 
        \STATE $\tilde{c}_t^{(s)} \leftarrow f_\phi\!\left(|\hat{r}_t^{(s)}|,\; H_{\text{ref},t}^{(s)}\right)$ 
        \STATE $c_t^{(s)} \leftarrow \tilde{c}_t^{(s)} \,/\, \text{mean}_{t'}(\tilde{c}_{t'}^{(s)})$ 
    \ENDIF
    \STATE \textit{// Token-specific KL-regularized preference loss}
    \STATE $\mathcal{L} \leftarrow -\log\sigma\!\left(\sum_{t} c_t^{(c)}\, \hat{r}_t^{(c)} - \sum_{t} c_t^{(r)}\, \hat{r}_t^{(r)}\right)$
    \STATE Update $\theta$ and $\phi$ via $\nabla_{\theta, \phi} \mathcal{L}$
\ENDFOR
\end{algorithmic}
\end{algorithm}
 
\subsection{Two Complementary Signals for Token Credit}
\label{sec:signals}
 
The analysis above establishes that token credit should be extracted online. We now ask: what signals should inform the credit assignment?
 
The implicit reward $\hat{r}_t$ (Eq.~\ref{eq:implicit-reward}) provides a natural candidate. However, a high $|\hat{r}_t|$ can arise from multiple causes: the model may have learned a meaningful preference at that position, or the reference model may simply be uncertain there, making the log-ratio noisy. Using $|\hat{r}_t|$ alone conflates signal and noise.
 
We observe that the reference model's token-level entropy provides a complementary signal:
\begin{equation}
    H_{\text{ref},t} = -\sum_{v \in \mathcal{V}} \pi_{\text{ref}}(v \mid x, y_{<t}) \log \pi_{\text{ref}}(v \mid x, y_{<t}),
    \label{eq:ref-entropy}
\end{equation}
where $\mathcal{V}$ is the vocabulary. At positions where $H_{\text{ref},t}$ is high, the reference model is uncertain, and the implicit reward is inherently noisier; at positions where $H_{\text{ref},t}$ is low, the reference model is confident, and deviations from it carry stronger signal.
 
This decomposition connects to the following result on optimal credit:
\begin{proposition}
\label{prop:optimal-weight}
Consider the preference logit $\hat{\Delta} = \sum_t c_t \hat{r}_t$, where each $\hat{r}_t = r_t^* + \epsilon_t$ is a noisy estimate of the true token reward with $\mathrm{Var}[\epsilon_t] = \sigma_t^2$. Since $c_t = \beta/\beta_t > 0$ by construction, we require $c_t \geq 0$. Assume the noise is $m$-dependent: $\mathrm{Cov}[\epsilon_t, \epsilon_{t'}] = 0$ for $|t - t'| > m$.

\textbf{(a)} Under independent noise ($m = 0$), the non-negative credits that minimize $\mathrm{Var}[\hat{\Delta}]$ subject to $\sum_t c_t |r_t^*| = C$ satisfy:
\begin{equation}
    c_t^* \;\propto\; \frac{|r_t^*|}{\sigma_t^2}.
    \label{eq:optimal-credit}
\end{equation}

\textbf{(b)} For $m$-dependent noise with bounded credits $c_t \in [c_{\min}, c_{\max}]$, the cross-covariance contribution to $\mathrm{Var}[\hat{\Delta}]$ is at most $O(m)$ times the diagonal term $\sum_t c_t^2 \sigma_t^2$. The solution from \textbf{(a)} thus remains a valid first-order approximation when $m \ll T$.
\end{proposition}
 
\begin{proof}[Proof sketch]
For Part~(a), we minimize $\mathrm{Var}[\hat{\Delta}] = \sum_t c_t^2 \sigma_t^2$ subject to the signal-preservation constraint $\sum_t c_t |r_t^*| = C$ with $c_t \geq 0$. The Lagrangian first-order condition $\partial \mathcal{L}/\partial c_t = 2c_t\sigma_t^2 - \lambda|r_t^*| = 0$ directly yields $c_t^* = \lambda|r_t^*|/(2\sigma_t^2) \propto |r_t^*|/\sigma_t^2$, which is automatically non-negative. For Part~(b), with bounded credits and $m$-dependent noise, the cross-covariance contribution to $\mathrm{Var}[\hat{\Delta}]$ scales as $O(m \cdot c_{\max}^2 \rho_{\max})$ relative to the diagonal term, remaining a bounded perturbation when $m \ll T$. The full proof is in Appendix~\ref{app:proof}.
\end{proof}

Part~(b) ensures that the conclusion is robust to the local correlations inherent in autoregressive generation, provided the credit values remain bounded---a property guaranteed by the mean normalization. We verify empirically that the noise autocorrelation drops below 0.05 within lag~1, yielding $m/T < 0.004$ (Figure~\ref{fig:validation}a), confirming that the independent-noise solution from Part~(a) is itself a close approximation.

We use $|\hat{r}_t|$ as a proxy for $|r_t^*|$. For $\sigma_t^2$, we observe that reference entropy $H_{\text{ref},t}$ provides an informative proxy: at high-entropy positions the reference model distributes probability mass across many tokens, so the log-ratio $\log\pi_\theta / \pi_{\text{ref}}$ is inherently more sensitive to small changes in $\pi_\theta$, increasing the estimation noise of $\hat{r}_t$. Empirically, we estimate $\sigma_t^2$ as the variance of $\hat{r}_t$ across the latter half of training checkpoints and find Spearman $\rho = 0.86$ between $H_{\text{ref},t}$ and $\hat{\sigma}_t^2$ (Figure~\ref{fig:validation}b--c), confirming this relationship. However, a binned analysis controlling for $|\hat{r}_t|$ reveals that the correlation is strongest at low-reward positions and diminishes at high-reward positions, indicating a nonlinear interaction between the two signals. This motivates a calibration network that learns the mapping jointly (Section~\ref{sec:credit-fn}), rather than fixing the functional form $c_t \propto |\hat{r}_t| / H_{\text{ref},t}$.

\paragraph{On potential-based shaping.}
The token-level implicit reward is not unique under arbitrary potential-based
shaping~\citep{rafailov2024token}. Se-DPO does not claim invariance to all such
decompositions. Instead, it operates under the canonical token decomposition
induced by the DPO log-ratio in Eq.~\ref{eq:implicit-reward}, which is
reproducible for a fixed policy--reference pair and provides a practical credit
signal. Our ablations in Section~\ref{sec:ablations} support this convention
empirically: using $|\hat{r}_t|$ alone already improves over uniform DPO, while
the large gap between Static-Credit and Se-DPO shows that tracking the evolution
of this canonical signal is more important than merely using a fixed token-level
decomposition.

\subsection{Credit Calibration Network}
\label{sec:credit-fn}
 
Proposition~\ref{prop:optimal-weight} suggests credit of the form $c_t \propto |\hat{r}_t| / H_{\text{ref},t}$, but the proxy relationships $|r_t^*| \approx |\hat{r}_t|$ and $\sigma_t^2 \approx H_{\text{ref},t}$ are approximate and interact nonlinearly (Section~\ref{sec:signals}). Moreover, the closed-form ratio is numerically unstable at low-entropy positions where $\sigma_t^2 \to 0$, producing arbitrarily large credits for tokens with negligible reward. To account for these issues, we parameterize the credit function 
as a lightweight MLP $f_\phi$ that maps the two signals to a 
scalar credit:
\begin{equation}
    \tilde{c}_t^{(s)} = f_\phi\!\left(|\hat{r}_t^{(s)}|,\; 
    H_{\text{ref},t}^{(s)}\right), \quad s \in \{c, r\},
    \label{eq:credit-fn}
\end{equation}
where $\phi$ is shared across all token positions and both 
response sides (architecture details in 
Appendix~\ref{app:details}).
 
Without constraint, the credit function could degenerately shrink 
all credits toward zero, which would trivially reduce the DPO 
loss. To prevent this, we normalize credits to have unit mean 
within each response, matching the total credit of standard DPO 
while allowing redistribution across tokens.
 
\subsection{Se-DPO: Complete Algorithm}
\label{sec:Se-DPO}
We now describe the complete Se-DPO algorithm, which integrates the components 
above. During the first $W$ steps (the \emph{warmup phase}), Se-DPO uses standard 
DPO with uniform credit ($c_t = 1$), avoiding amplification of early noise when 
$|\hat{r}_t| \approx 0$ for all tokens. From step 
$W{+}1$ onward (the \emph{guided phase}), at each training step Se-DPO computes 
$|\hat{r}_t^{(s)}|$ and $H_{\text{ref},t}^{(s)}$ from the current forward pass 
through the policy and reference models. Since DPO already evaluates both models on each batch, these quantities can be derived from the existing forward pass with minimal overhead. The calibration network $f_\phi$ then maps $|\hat{r}_t^{(s)}|$ and $H_{\text{ref},t}^{(s)}$ to token credit via  Eqs.~\ref{eq:credit-fn}.
 
\paragraph{Loss.} The token credit modulates each token's contribution to the preference loss:
\begin{equation}
    \mathcal{L}_{\text{Se-DPO}} = -\log\sigma\!\left(\sum_{t} c_t^{(c)}\, \hat{r}_t^{(c)} - \sum_{t} c_t^{(r)}\, \hat{r}_t^{(r)}\right),
    \label{eq:Se-DPO-loss}
\end{equation}

Se-DPO adds only the forward and backward pass of the calibration 
network per training step. Both input signals---$|\hat{r}_t|$ and 
$H_{\text{ref},t}$---are derived from the standard DPO forward 
pass, requiring no additional model evaluations.

\begin{table}[t]
\centering
\caption{Experiment results on AlpacaEval~2 \citep{alpacaeval}, Arena-Hard \citep{arenahard}, and MT-Bench benchmarks. Best results in each setting are \textbf{bolded}.}
\label{tab:main}
\scriptsize
\resizebox{\textwidth}{!}{
\begin{tabular}{lcccccccc}
\toprule
 & \multicolumn{4}{c}{\textbf{Llama-3-8B-Instruct (PairRM)}} & \multicolumn{4}{c}{\textbf{Llama-3-8B-Instruct (ArmoRM)}} \\
\cmidrule(lr){2-5} \cmidrule(lr){6-9}
\multirow{2}{*}{\textbf{Method}} & \textbf{AlpacaEval 2} & \textbf{Arena-Hard} & \multicolumn{2}{c}{\textbf{MT-Bench}} & \textbf{AlpacaEval 2} & \textbf{Arena-Hard} & \multicolumn{2}{c}{\textbf{MT-Bench}} \\
\cmidrule(lr){2-2} \cmidrule(lr){3-3} \cmidrule(lr){4-5} \cmidrule(lr){6-6} \cmidrule(lr){7-7} \cmidrule(lr){8-9}
 & \textbf{WR (\%)} & \textbf{WR (\%)} & \textbf{Score} & \textbf{WR (\%)} & \textbf{WR (\%)} & \textbf{WR (\%)} & \textbf{Score} & \textbf{WR (\%)} \\
\midrule
SFT& 30.6 & 21.4 & 7.9 & 27.5 & 30.6 & 21.4 & 7.9 & 27.5 \\
DPO& 41.7 & 30.4 & 8.0 & 37.5 & 40.8 & 36.2 & \textbf{8.2} & \textbf{46.3} \\
SimPO& 39.8 & 28.7 & 7.8 & 32.5 & 37.0 & 28.1 & 7.8 & 42.5 \\
Static-Credit& 42.3 & 31.2 & 8.0 & 37.4 & 41.2 & 38.3 & 8.1 & 45.6 \\
TGDPO& 43.9 & 34.3 & 8.0 & \textbf{41.9} & 42.5 & 40.5 & 7.9 & 45.0 \\
Se-DPO & \textbf{47.2} & \textbf{42.6} & 7.4 & \textbf{41.9} & \textbf{50.6} & \textbf{43.3} & 6.9 & 40.0 \\
\midrule
 & \multicolumn{4}{c}{\textbf{Llama-3.2-3B-Instruct (ArmoRM)}} & \multicolumn{4}{c}{\textbf{Gemma-2-2B-it (ArmoRM)}} \\
\cmidrule(lr){2-5} \cmidrule(lr){6-9}
\multirow{2}{*}{\textbf{Method}} & \textbf{AlpacaEval 2} & \textbf{Arena-Hard} & \multicolumn{2}{c}{\textbf{MT-Bench}} & \textbf{AlpacaEval 2} & \textbf{Arena-Hard} & \multicolumn{2}{c}{\textbf{MT-Bench}} \\
\cmidrule(lr){2-2} \cmidrule(lr){3-3} \cmidrule(lr){4-5} \cmidrule(lr){6-6} \cmidrule(lr){7-7} \cmidrule(lr){8-9}
 & \textbf{WR (\%)} & \textbf{WR (\%)} & \textbf{Score} & \textbf{WR (\%)} & \textbf{WR (\%)} & \textbf{WR (\%)} & \textbf{Score} & \textbf{WR (\%)} \\
\midrule
SFT& 23.8 & 17.1 & 7.0 & 16.3 & 32.8 & 20.1 & 7.9 & 37.5 \\
DPO& 29.6 & 23.2 & 7.9 & 29.4 & 40.8 & 26.4 & 8.0 & 43.1 \\
SimPO& 26.2 & 22.6 & 7.4 & 15.7 & 34.8 & 21.1 & 7.8 & 40.0 \\
Static-Credit& 30.1 & 23.8 & 7.8 & 28.8 & 41.2 & 27.3 & 7.9 & 42.5 \\
TGDPO& \textbf{35.8} & 25.4 & 8.1 & 36.9 & 43.0 & \textbf{30.7} & \textbf{8.1} & \textbf{46.9} \\
Se-DPO(LoRA) & 34.2 & \textbf{28.8} & \textbf{8.3} & \textbf{42.5} & \textbf{44.4} & 29.6 & 7.9 & 42.3 \\
\bottomrule
\end{tabular}
}
\end{table}

\section{Experiments}
\label{sec:experiments}
\subsection{Setup}
\label{sec:setup}

\paragraph{Models and data.} We conduct experiments on three models: Llama-3.2-3B-Instruct, Llama-3-8B-Instruct \citep{llama3}, and Gemma-2-2B-it \citep{gemma2}. Following \citet{tgdpo} and \citet{simpo}, we use prompts from the UltraFeedback dataset \citep{ultrafeedback} and let each model generate 5 responses with a temperature of 0.8. These responses are then ranked using ArmoRM \citep{armorm}, with the highest and lowest-ranked responses selected as the chosen and rejected samples, respectively. For Llama-3-8B-Instruct, we additionally use PairRM \citep{pairrm} for preference annotation to evaluate robustness across different annotators. The data construction protocol is summarized in Appendix Table~\ref{tab:data_protocol}. Each dataset contains approximately 60K preference pairs. We compare against DPO \citep{dpo}, SimPO \citep{simpo}, TGDPO \citep{tgdpo}, and
Static-Credit, using the original method-specific settings whenever applicable. The SFT row refers to the original instruction-tuned model evaluated directly, without additional supervised fine-tuning. All methods are trained for 1 epoch with the AdamW optimizer and a cosine learning rate schedule with 10\% warmup. Hyperparameter settings are presented in Appendix~\ref{app:details}.

\paragraph{Evaluation benchmarks.} We use three prominent instruction-following benchmarks for evaluation: AlpacaEval~2 \citep{alpacaeval}, Arena-Hard \citep{arenahard}, and MT-Bench. For AlpacaEval~2, we report the win rate against GPT-4 Turbo. For Arena-Hard, we report the win rate against GPT-4-0314. For MT-Bench, we report the score and win rate against GPT-4. 

\begin{table}[t]
\centering
\caption{Warmup sensitivity on Llama-3.2-3B-Instruct with ArmoRM annotation.}
\label{tab:warmup_3b}
\small
\begin{tabular*}{\textwidth}{@{\extracolsep{\fill}}lccc}
\toprule
Variant & AlpacaEval~2 WR (\%) & AlpacaEval~2 LC WR (\%) & Arena-Hard WR (\%) \\
\midrule
No warmup & 32.57 & 33.03 & 22.1 \\
20\% warmup & 31.15 & 31.50 & 22.8 \\
50\% warmup & 32.08 & 32.53 & 25.2 \\
4\% warmup (default) & 34.20 & 32.60 & 28.8 \\
\bottomrule
\end{tabular*}
\end{table}

\subsection{Main Results}

Table~\ref{tab:main} summarizes the results across three base models and two preference annotators. Se-DPO achieves the strongest or near-strongest performance on AlpacaEval~2 and Arena-Hard across most settings. On Llama-3-8B-Instruct with ArmoRM annotation, Se-DPO achieves 50.6\% win rate on AlpacaEval~2 and 43.3\% on Arena-Hard, surpassing the strongest prior baseline in Table~\ref{tab:main} by +8.1 and +2.8 points; with PairRM annotation, Se-DPO similarly leads with 47.2\% and 42.6\%, improving over the strongest prior baseline in Table~\ref{tab:main} by +3.3 and +8.3 points. Notably, these gains are achieved without any external model. Se-DPO also shows consistent improvements on AlpacaEval~2 and Arena-Hard under both PairRM and ArmoRM annotations, suggesting that the self-evolving credit mechanism is robust to the choice of preference annotator, as it derives importance signals entirely from the model's own internal states.

On smaller models (Llama-3.2-3B-Instruct and Gemma-2-2B-it), Se-DPO is trained with LoRA in the main table. Se-DPO (LoRA) consistently improves over standard DPO on AlpacaEval~2 and Arena-Hard, while remaining competitive with TGDPO across smaller-model settings. The margins are narrower than in the Llama-3-8B full fine-tuning setting, which suggests that parameter-efficient training may limit how precisely the policy can realize fine-grained, position-dependent deviations from the reference policy.

\subsection{Ablation Studies}
\label{sec:ablations}

We conduct component ablations on Llama-3.2-3B-Instruct with ArmoRM annotation to isolate the contribution of online updating, the two input signals, warmup, and learned calibration (Figure~\ref{fig:ablation}). Static-Credit computes token credit once at the end of warmup and then keeps it frozen. It improves only modestly over DPO, confirming that a fixed per-token credit signal is useful but insufficient. In contrast, Se-DPO substantially outperforms Static-Credit, directly supporting our central claim that token credit should evolve with the policy during training.


\begin{figure*}[t]
\centering
\includegraphics[width=\textwidth]{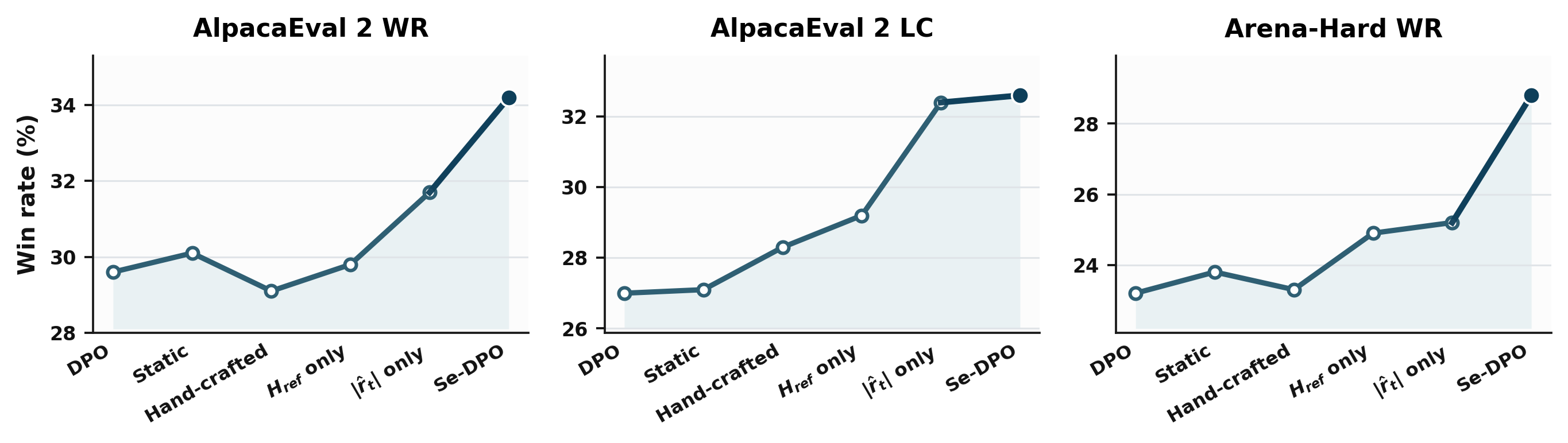}
\caption{Ablation results on Llama-3.2-3B-Instruct. We compare DPO, Static-Credit, hand-crafted credit, single-signal variants, and Se-DPO across AlpacaEval~2 raw win rate, AlpacaEval~2 length-controlled win rate, and Arena-Hard win rate.}
\label{fig:ablation}
\end{figure*}

The input ablations show that both implicit reward magnitude and reference entropy carry useful information. The hand-crafted ratio does not consistently improve over the single-signal variants, indicating that a fixed functional form is insufficient for capturing the relationship between reward strength and noise. This is consistent with our analysis in Section~\ref{sec:signals}: reference entropy is most informative at low-reward positions, while high-reward positions require a more regime-dependent mapping. We also find that high-credit tokens at convergence have near-zero rank correlation with the variance of $\hat{r}_t$ across the final checkpoints, suggesting that the calibration network does not simply amplify noisy tokens.

\paragraph{Warmup sensitivity.}
The Se-DPO credit warmup $W$ is independent of the optimizer learning-rate warmup. During the first $W$ steps, training uses the standard DPO loss and the calibration network is inactive; after $W$ steps, credit-based KL allocation is enabled and $f_\phi$ is updated jointly with the policy. Table~\ref{tab:warmup_3b} shows that the default short warmup performs best on AlpacaEval~2 raw win rate and Arena-Hard, while remaining comparable on length-controlled AlpacaEval~2. Removing warmup or delaying credit activation reduces performance on these primary win-rate benchmarks, suggesting that Se-DPO benefits from a brief stabilization phase before online credit updates begin, while overly long warmups reduce the benefit of online credit adaptation.



Because LLM-as-a-judge benchmarks can be sensitive to response length, we report length-controlled AlpacaEval~2 scores and mean generation lengths in Appendix Table~\ref{tab:length}. On Llama-3-8B-Instruct, Se-DPO's average length is higher than DPO, but it retains a large length-controlled advantage over DPO and TGDPO, indicating that the gains are not explained by verbosity alone. We also observe that Se-DPO's MT-Bench score is lower than DPO in the Llama-3-8B full fine-tuning settings, and its MT-Bench win rate varies across annotators. Since MT-Bench score averages absolute ratings while win rate counts pairwise victories, improvements on pairwise instruction-following benchmarks may not always translate directly to MT-Bench scores; we leave this trade-off for future investigation.


\section{Related Work}
\label{sec:related}
\paragraph{Reinforcement learning from human feedback.}
RLHF has become the dominant paradigm for aligning large language models with human preferences \citep{ouyang2022training,stiennon2020learning,schulman2017ppo}. The standard pipeline first trains a reward model from human comparisons, then optimizes the policy via PPO. DPO \citep{dpo} simplifies this pipeline by reparameterizing the reward into a supervised loss over preference pairs, eliminating the need for a separate reward model. This formulation has inspired a family of variants that modify the optimization objective along different axes: removing the reference model in favor of length-normalized implicit rewards \citep{simpo}, replacing pairwise preferences with prospect-theoretic pointwise feedback \citep{kto}, adopting alternative divergence constraints \citep{ipo}, decomposing optimization to the reasoning-step level \citep{wang-etal-2025-explore, stepdpo}, or refining the optimization under autoregressive or small-margin settings \citep{adpo,mixdpo}. Orthogonally, the KL penalty coefficient $\beta$ has been made adaptive at the sample level---via per-pair perturbation \citep{epsilon-dpo}, batch-level adjustment based on data quality \citep{beta-dpo}, instance-level weighting of preference pairs \citep{wpo}, or adaptive batch-wise scheduling \citep{sams}---but these methods modulate regularization per sample rather than per token. Se-DPO extends adaptive regularization to token granularity by deriving per-token KL budgets from the model's own evolving signals.

\begin{table}[t]
\centering
\caption{Comparison of fine-grained preference optimization methods.}
\label{tab:comparison}
\small
\begin{tabular*}{\textwidth}{@{\extracolsep{\fill}}lcccc}
\toprule
Method & Granularity & Signal Source & Extra Models? & Extra Cost \\
\midrule
TGDPO & Token & Teacher model & Yes & Teacher training \\
TIS-DPO & Token & Contrastive LLMs & Yes & 2 extra LLMs \\
SparsePO & Token & Loss-driven masks & No & $O(T)$ mask params \\
OTPO & Token & OT on embeddings & No & OT solver \\
Q-RM & Token & Discriminative model & Yes & Full model training \\
\midrule
$\varepsilon$-DPO & Instance & Logit perturbation & No & Negligible \\
DICE & Sequence & Implicit reward & No & Multi-round training \\
\midrule
Se-DPO (ours) & Token & Implicit reward & No & Minimal \\
\bottomrule
\end{tabular*}
\end{table}

\paragraph{Token-level rewards for preference optimization.}
Standard DPO aggregates token-level log-ratios via uniform summation, yet human judgments often hinge on a small subset of tokens. A growing body of work addresses this mismatch by introducing token-level importance signals into the preference loss. Several approaches estimate token importance from external models such as pre-trained teachers, contrastive LLM pairs, or discriminative Q-function models \citep{tgdpo,tisdpo,treg,qrm}. Others avoid external models but introduce auxiliary mechanisms---learned sparse masks, optimal transport, per-token KL constraints, gradient attribution, or token-adaptive barriers \citep{sparsepo,otpo,tdpo,tidpo,tabpo}. Separately, DPO's implicit reward has been exploited at the sequence level for iterative self-alignment, token selection, PPO-based fine-tuning, process rewards, and token-level distillation \citep{dice,sepo,rto,prime,aligndistil}. More broadly, iterative and feedback-driven optimization has been studied in both language-model alignment and scientific discovery settings \citep{dice,prime,zhao2026asr,li2026autosupervision}. Learning from ranking or preference signals also connects to broader work on robust optimization under noisy or imbalanced supervision \citep{ipo,kto,yu2026rubric}. Across these directions, token importance is generally computed outside the training loop or extracted from a fixed checkpoint. Table~\ref{tab:comparison} summarizes this comparison. Se-DPO differs in that it tracks the implicit reward online as it evolves during training and calibrates it with reference entropy for noise robustness, without requiring any external models.

\section{Conclusion}

We have identified a previously overlooked temporal dimension of token-level preference optimization: the implicit reward learned by DPO evolves substantially during training, with only 56\% overlap between top-$|\hat r_t|$ tokens at early and late stages. This finding challenges the static importance assumption shared by existing token-level methods, which compute token-level importance signals prior to or outside the training loop. We formalized the notion of token credit as per-token KL budget allocation, and showed via a variance-minimization analysis that optimal credit should be proportional to implicit reward magnitude and inversely related to estimation noise. Since both quantities change as the model trains, effective credit must evolve with the model.

Based on this analysis, we proposed Se-DPO, which derives token credit online from two complementary signals---implicit reward magnitude and reference entropy---through a lightweight calibration network. Se-DPO requires no external models, no pre-trained teachers, no contrastive LLM pairs, and no complex solvers, adding only minimal computational overhead to standard DPO training. Experiments across three base models and two preference annotators demonstrate consistent improvements on AlpacaEval~2 and Arena-Hard, with Se-DPO achieving up to 50.6\% win rate on AlpacaEval~2 and 43.3\% on Arena-Hard, surpassing most baselines without external models.

Our analysis builds on two proxy relationships: $|\hat{r}_t|$ as a stand-in for the true reward magnitude and $H_{\text{ref},t}$ as an indicator of estimation noise. The two design choices in Se-DPO---the warmup phase and the calibration network---are specifically introduced to account for their approximate nature: the warmup phase ensures that $|\hat{r}_t|$ has stabilized before being used as credit, while the calibration network learns to handle the nonlinear interaction between the two signals rather than relying on a fixed functional form. Incorporating richer representations from the model's hidden states could potentially improve credit estimation. 

\bibliography{colm2026_conference}
\bibliographystyle{colm2026_conference}
 
\newpage
\appendix
\section{Proof of Proposition~\ref{prop:optimal-weight}}
\label{app:proof}

\paragraph{Part (a).}
Since $c_t = \beta/\beta_t > 0$, we restrict to non-negative credits. Under independent noise, $\text{Var}[\hat{\Delta}] = \sum_{t} c_t^2 \sigma_t^2$. We minimize this subject to $\sum_{t} c_t |r_t^*| = C$ with $c_t \geq 0$. The Lagrangian is
\[
\mathcal{L} = \sum_t c_t^2 \sigma_t^2 - \lambda\!\left(\sum_t c_t |r_t^*| - C\right).
\]
Setting $\partial \mathcal{L} / \partial c_t = 0$ yields $2 c_t \sigma_t^2 = \lambda |r_t^*|$, so
\[
c_t^* = \frac{\lambda\, |r_t^*|}{2\sigma_t^2} \;\geq\; 0,
\]
which is automatically non-negative. Since credits are normalized to unit mean, the proportionality constant cancels:
$c_t^* \propto |r_t^*| / \sigma_t^2$.

\paragraph{Part (b).}
For $m$-dependent noise, the full variance is
\[
\text{Var}[\hat{\Delta}] = \underbrace{\sum_t c_t^2 \sigma_t^2}_{V_{\text{diag}}} + \underbrace{\sum_{t}\sum_{\substack{t':\, 0 < |t-t'| \leq m}} c_t c_{t'} \text{Cov}[\epsilon_t, \epsilon_{t'}]}_{V_{\text{cross}}}.
\]
With bounded credits $c_t \in [c_{\min}, c_{\max}]$ and writing $\rho_{\max} = \max_{|t-t'| \leq m} |\text{Corr}[\epsilon_t, \epsilon_{t'}]|$, the cross term satisfies
\[
|V_{\text{cross}}| \leq 2m \cdot c_{\max}^2 \cdot \rho_{\max} \sum_t \sigma_t^2.
\]
Meanwhile $V_{\text{diag}} \geq c_{\min}^2 \sum_t \sigma_t^2$. Hence $|V_{\text{cross}}| / V_{\text{diag}} \leq 2m\, \rho_{\max}\, (c_{\max}/c_{\min})^2$. When $m \ll T$ and the credit ratio $c_{\max}/c_{\min}$ is moderate (as enforced by mean normalization), the cross term is a bounded perturbation and the solution from Part~(a) remains a valid first-order approximation.

\section{Additional Experimental Results}
\label{app:additional-results}

\subsection{Training Cost}

Se-DPO adds only a lightweight calibration network and reference-entropy
computation on top of the standard DPO forward pass. Table~\ref{tab:time_3b}
reports the wall-clock training time on Llama-3.2-3B-Instruct with ArmoRM
annotation.


\begin{table}[h]
\centering
\caption{Wall-clock training time on Llama-3.2-3B-Instruct with ArmoRM annotation.}
\label{tab:time_3b}
\small
\begin{tabular}{lcc}
\toprule
Method & Training Time & Relative Overhead \\
\midrule
DPO & 97.1 min & -- \\
Se-DPO & 103.7 min & +6.85\% \\
\bottomrule
\end{tabular}
\end{table}

The overhead comes from the lightweight calibration network and reference
entropy computation. Since both $|\hat{r}_t|$ and $H_{\mathrm{ref},t}$ are
derived from the standard DPO forward pass without additional model evaluations,
this cost is substantially lower than methods requiring teacher, contrastive, or
separate discriminative reward models.





\begin{table}[t]
\centering
\caption{Length analysis across Llama-3-8B-Instruct and Llama-3.2-3B-Instruct with ArmoRM annotation. AE2 LC is the length-controlled AlpacaEval~2 win rate.}
\label{tab:length}
\small
\begin{tabular*}{\textwidth}{@{\extracolsep{\fill}}llcccc}
\toprule
Model & Method & AE2 WR (\%) & AE2 LC (\%) & AE2 Avg. L & AH Avg. L \\
\midrule
Llama-3-8B-Instruct & DPO & 40.8 & 38.2 & 1837 & 621 \\
Llama-3-8B-Instruct & SimPO & 37.0 & 41.2 & 1825 & 583 \\
Llama-3-8B-Instruct & TGDPO & 42.5 & 40.1 & 1882 & 638 \\
Llama-3-8B-Instruct & Se-DPO & 50.6 & 47.4 & 2215 & 661 \\
\midrule
Llama-3.2-3B-Instruct & Se-DPO LoRA & 34.2 & 32.6 & 2005 & 414 \\
\bottomrule
\end{tabular*}
\end{table}

\subsection{MT-Bench Category Breakdown}

Table~\ref{tab:mtbench-breakdown} reports the category-level MT-Bench scores for
Se-DPO on Llama-3-8B-Instruct with PairRM annotation. Se-DPO performs strongly
on open-ended generation categories, while the lower overall mean is driven by
Reasoning, Coding, and Math. This explains why MT-Bench score and MT-Bench win
rate can diverge: the score averages absolute ratings, whereas win rate counts
pairwise victories across questions.

\begin{table}[h]
\centering
\caption{MT-Bench category breakdown for Se-DPO on Llama-3-8B-Instruct with PairRM annotation.}
\label{tab:mtbench-breakdown}
\small
\begin{tabular}{lc}
\toprule
Category & Score \\
\midrule
Writing & 9.62 \\
Roleplay & 9.40 \\
STEM & 9.47 \\
Humanities & 9.57 \\
Extraction & 8.10 \\
Reasoning & 4.50 \\
Coding & 4.60 \\
Math & 3.90 \\
\bottomrule
\end{tabular}
\end{table}

\subsection{Token-Credit Behavior}

We examine representative evaluation samples to understand the behavior of the
learned credit. The resulting per-token KL budgets are context-dependent: in
code-related responses, instruction tokens and variable names tend to receive larger credit
than comment markers and section labels; in creative writing, character names
and narrative-driving tokens tend to receive larger credit; in poetry, line
breaks can receive elevated credit because they carry structural information.
Quantitatively, the learned credit distribution is highly non-uniform:
top-credited positions account for a larger share of the implicit-reward
contribution and receive correspondingly larger KL budgets, while low-credit
positions contribute less to the preference signal and are kept under tighter
regularization. This supports the view that the
calibration network learns to allocate token-level KL budgets according to
informative implicit-reward patterns rather than applying uniform token-level
regularization.

\section{Additional Experimental Details}
\label{app:details}

\begin{table}[t]
\centering
\caption{Preference-data construction and annotation protocol.}
\label{tab:data_protocol}
\small
\begin{tabular*}{\textwidth}{@{\extracolsep{\fill}}ll}
\toprule
Component & Setting \\
\midrule
Prompt source & UltraFeedback prompts \\
Candidate responses & 5 responses per prompt generated by the target instruction-tuned model \\
Sampling temperature & 0.8 \\
Preference annotator & ArmoRM; PairRM additionally for Llama-3-8B-Instruct \\
Chosen / rejected pair & Highest-ranked response vs. lowest-ranked response \\
Training size & Approximately 60K preference pairs per setting \\
\bottomrule
\end{tabular*}
\end{table}

\paragraph{Preliminary analysis (Section~\ref{sec:motivation}).} The token importance evolution analysis in Section~\ref{sec:motivation} is conducted on Qwen2.5-1.5B-Instruct with $\beta=0.1$, learning rate $10^{-5}$, trained for 1 epoch on 2{,}000 samples from UltraFeedback. We track 200 diagnostic samples across 15 evenly spaced checkpoints. The Se-DPO LoRA credit-dynamics probe in Figure~\ref{fig:credit_dynamics}(c--f) is conducted separately on fixed evaluation samples and is used only as a mechanism analysis of learned credit dynamics. The noise autocorrelation in Figure~\ref{fig:validation}(a) is computed at three training stages (7\%, 56\%, 99\% progress). The empirical reward variance $\hat{\sigma}_t^2$ in Figure~\ref{fig:validation}(b--c) is estimated as the variance of $\hat{r}_t$ across the latter half of training checkpoints.

\paragraph{Credit calibration network.} The calibration network $f_\phi$ is a two-layer MLP: Linear(2, 16) $\to$ ReLU $\to$ Linear(16, 1) $\to$ Softplus. The Softplus activation ensures non-negative output. The learning rate for $\phi$ is set to $10^{-3}$, higher than the policy learning rate, to enable fast adaptation to the evolving implicit reward signal.

\paragraph{Training hyperparameters.} All methods are trained for 1 epoch with the AdamW optimizer and a cosine learning rate schedule with 10\% warmup ratio. For Llama-3-8B-Instruct, we use a learning rate of $5 \times 10^{-7}$, $\beta = 0.01$. For Llama-3.2-3B-Instruct and Gemma-2-2B-it, Se-DPO is trained with LoRA (rank 64, $\alpha = 128$, applied to all linear layers), with a learning rate of $5 \times 10^{-6}$ and $\beta = 0.1$. The Se-DPO credit warmup $W$ is set to approximately 4\% of total training steps in the main experiments. DPO, SimPO, and TGDPO follow the hyperparameters reported in \citet{tgdpo}; additional token-level baselines use their original method-specific settings whenever applicable.

\end{document}